\documentclass[runningheads, orivec]{llncs}
\usepackage[T1]{fontenc}
\usepackage{graphicx,verbatim}
\usepackage{subcaption}
\usepackage{booktabs}
\usepackage{makecell}
\usepackage{amsfonts}
\usepackage{amsmath}
\usepackage{multirow}
\usepackage[table]{xcolor}
\usepackage[x11names,svgnames,dvipsnames]{xcolor}

\usepackage{amssymb}
\usepackage{nicefrac}

\usepackage{hyperref}

\definecolor{myred}{RGB}{255,150,150} 
\definecolor{mygreen}{RGB}{150,255,150} 
\definecolor{myyellow}{RGB}{255,255,0}

\usepackage{color}
\usepackage[style=nature, natbib=true, sorting=nyt, maxbibnames=1, maxcitenames=2, uniquelist=false, backend=biber]{biblatex}
\DeclareFieldFormat{title}{#1}
\DeclareFieldFormat{booktitle}{#1}
\DeclareFieldFormat{journaltitle}{#1}
\renewbibmacro*{in:}{%
  \printtext{%
    In\intitlepunct}}
\AtBeginBibliography{\small}
\begin{document}
\title{Fairness Beyond Demographics: Optimizing Performance Across Appearance‑Based Hidden Cohorts in Medical Imaging}
%Learning to Mitigate Hidden Cohort Bias in Medical Imaging}
\titlerunning{Fairness Beyond Demographics}
% If the paper title is too long for the running head, you can set
% an abbreviated paper title here
%
\begin{comment}  %% Removed for anonymized MICCAI submission
\author{First Author\inst{1}\orcidID{0000-1111-2222-3333} \and
Second Author\inst{2,3}\orcidID{1111-2222-3333-4444} \and
Third Author\inst{3}\orcidID{2222--3333-4444-5555}}
%
\authorrunning{F. Author et al.}
% First names are abbreviated in the running head.
% If there are more than two authors, 'et al.' is used.
%
\institute{Princeton University, Princeton NJ 08544, USA \and
Springer Heidelberg, Tiergartenstr. 17, 69121 Heidelberg, Germany
\email{lncs@springer.com}\\
\url{http://www.springer.com/gp/computer-science/lncs} \and
ABC Institute, Rupert-Karls-University Heidelberg, Heidelberg, Germany\\
\email{\{abc,lncs\}@uni-heidelberg.de}}

\end{comment}

\author{Milad Masroor \and Cuong Nguyen \and Kevin Wells \and Gustavo Carneiro}  %% Added for anonymized MICCAI submission
\authorrunning{Masroor et al.}
\institute{Centre for Vision, Speech and Signal Processing, University of Surrey, Guildford, UK \\
    \email{\{m.masroor,c.nguyen,k.wells,g.carneiro\}@surrey.ac.uk}}
  
\maketitle              % typeset the header of the contribution
\begin{abstract}

Medical image analysis models can exhibit performance disparities across patient subgroups, threatening clinical safety and fairness.  Existing methods typically address this issue by optimizing accuracy and fairness metrics for visible demographic attributes (e.g., sex or age) considered in isolation.
This strategy not only overlooks potentially more informative latent stratifications, which may reveal deeper sources of model error and inequity, but also fails to scale when multiple demographic attributes are considered simultaneously due to the resulting sparsity of training data within each subgroup.
We deal with these issues by introducing the label-free hidden-cohort fairness (LHCF) training paradigm that instead of maximizing fairness over visible demographic attributes, it optimizes fairness across latent subpopulations discovered  from image appearance. By clustering images into $K$ appearance‑based cohorts and applying fairness optimization over them, LHCF uncovers underlying sources of model error and avoids the combinatorial sparsity of multi‑demographic attributes, reducing disparities across both single and multiple demographic attributes.
We demonstrate on our proposed fairness benchmark, HIDFairBench, that LHCF provides state-of-the-art fairness results on single and multiple demographic attributes, despite never using demographic labels for training. Our results position hidden‑cohort fairness as a practical, scalable, and robust alternative to demographic‑based fairness optimization for trustworthy medical image analysis. Code will be available upon paper acceptance.

% Authors must provide keywords and are not allowed to remove this Keyword section.

\end{abstract}

\section{Introduction}
\label{sec:intro}

Medical image analysis (MIA) systems often exhibit performance disparities across visible demographic attributes (e.g., sex, age, ethnicity)~\cite{seyyed2021underdiagnosis,christodoulou2024confidence,koch2024distribution,oakden2020hidden,finlayson2021clinician,bissoto2025subgroup}, raising concerns about clinical safety and fairness~\cite{seyyed2021underdiagnosis,christodoulou2024confidence}. Although these attributes are assumed to be the main drivers of such gaps, recent studies have shown that they may align only weakly with latent factors that may better explain model behavior~\cite{oakden2020hidden,bissoto2025subgroup,seo2022unsupervised}. Indeed, performance disparities can be much larger between hidden patient cohorts (formed by visually coherent but unlabeled subsets of the data) than across  demographic groups, creating  low‑performing subpopulations often referred to as underperforming slices or blind spots~\cite{wynants2020prediction,eyuboglu2022domino,plumb2022towards,olesen2024slicing}. This challenge becomes more severe when multiple attributes are considered jointly, as the number of intersectional subgroups expands while available samples per subgroup shrink~\cite{buolamwini2018gender}. These issues reflect the hidden stratification phenomenon, in which coarse labels fail to capture meaningful clinical variation~\cite{oakden2020hidden,zech2018variable}, causing models to perform poorly on underrepresented cohorts~\cite{degrave2021ai,sagawa2019distributionally,rajpurkar2018deep,sohoni2020no} and ultimately increasing the risk of disproportionate patient harm~\cite{hashimoto2018fairness}.

Despite growing evidence that hidden patient cohorts represent a major driver of model failures, fairness methods in medical imaging still focus primarily on visible demographic attributes~\cite{tian2024fairseg,luo2024fairclip,luo2024fairvisionequitabledeeplearning,tian2024fairdomain,masroor2024fair}, while work on hidden cohorts has mostly diagnosed, instead of mitigating, the reported disparities~\cite{bissoto2025subgroup,olesen2024slicing,collevati2024leveraging,chakraborty2024exmap}. Demographic‑focused fairness approaches also scale poorly as intersectional groups increase and per‑group sample sizes collapse~\cite{buolamwini2018gender}. 
%These observations motivate shifting fairness optimization toward a small number of latent cohorts with sufficient data. 
Taken together, these limitations point to latent cohorts, with sufficient data and richer structure, as a more reliable basis for fairness optimization than visible demographic groups. To illustrate this paradigm shift, we compare two approaches: (i) classic fairness training that optimizes across visible demographic attributes; and (ii) label-free hidden-cohort fairness (LHCF), which optimizes fairness across clusters formed with appearance-based embeddings. Fig.~\ref{fig:motivation_visible_complexity} shows that LHCF yields the best classification and fairness results not only for isolated demographic groups but also for multi‑attribute intersections.

In this paper, we propose LHCF (Fig.~\ref{fig:method_overview}), which optimizes fairness over appearance‑based latent cohorts without demographic labels and is compatible with existing fairness approaches.
We also propose the HIdden‑Demographic Fairness Benchmark (HIDFairBench), a new benchmark to stress-test fairness methods.
This paper has the following contributions: 
\begin{itemize}
    \item The new LHCF training that clusters images into $K$ hidden clusters and applies fairness optimization to balance accuracy and fairness across these clusters. This approach avoids the combinatorial sparsity of multi‑attribute demographic groups and reduces performance disparities not only for the hidden clusters but also across single and multiple demographic attributes.
    %we show that optimizing fairness over hidden cohorts leads to improved fairness on both demographic and multi‑demographic groups, despite not using demographic labels during training;
    %\item The novel label‑independent fairness training algorithm that avoids reliance on demographic metadata and remains stable under increasing cohort complexity;
    \item %The new HIDFairBench benchmark that analyses the quality of hidden cohorts found by LHCF and tests the performance of fairness methods under increasingly complex multi‑demographic partitions, revealing the scalability limits of demographic‑based approaches.
    The new HIDFairBench, a benchmark that evaluates hidden-cohort quality and tests fairness methods under increasingly complex multi‑demographic partitions, exposing the scalability limits of fairness approaches.
\end{itemize}
Our experiments show that LHCF achieves state‑of‑the‑art accuracy and fairness across single and multiple demographic attributes from HIDFairBench, despite never using demographic labels for training, establishing LHCF as a scalable and robust alternative to demographic‑based fairness optimization in MIA.

\begin{figure*}[t!]
    \centering
    
    % -------------------- (a) --------------------
    \begin{subfigure}{\textwidth}
        \centering
        \includegraphics[width=0.8\textwidth]{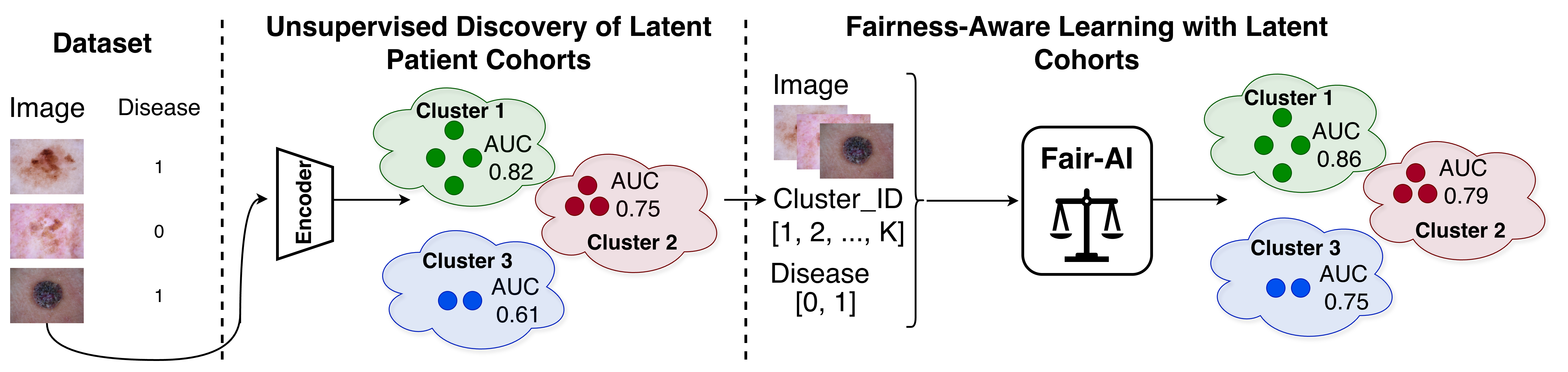}
        \caption{LHCF training.}
        \label{fig:method_overview}
    \end{subfigure}
    
    \vspace{0.5cm}
    
    % -------------------- (b) --------------------
    \begin{subfigure}{\textwidth}
        \centering
        \includegraphics[width=0.9\textwidth]{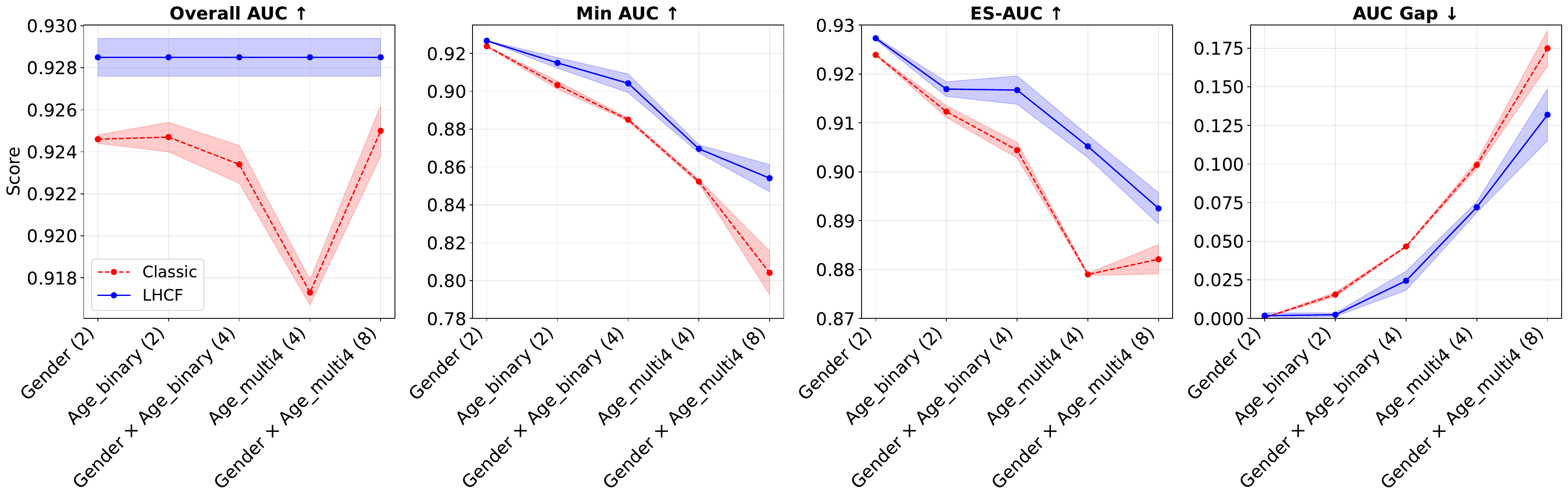}
        \caption{Classic vs. LHCF Results.}
        \label{fig:motivation_visible_complexity}
    \end{subfigure}
    
    \caption{
    (a) LHCF training. 
    (b) Overall AUC (left) and fairness metrics (Min AUC, ES-AUC, AUC Gap; second-to-fourth plots) across visible cohort complexity for Classic demographic fairness training (i) and LHCF training (ii) of FairDi~\cite{masroor2024fair} (using ResNet18 backbone) on HAM10000~\cite{tschandl2018ham10000}. Visible-cohort complexity increases from Gender (2 groups) to Age\_binary (2 groups), Age\_multi4 (4 groups), and intersectional cohorts Gender$\times$Age\_binary (4 groups) and Gender$\times$Age\_multi4 (8 groups) -- Sec.~\ref{sec:benchmark} explains this experimental setup.
    %Classic degrades as complexity grows, whereas LHCF maintains superior performance.
    }
    
\end{figure*}

\section{Method}
\label{sec:method}

In this section, we present the LHCF training and the HIDFairBench benchmark.

\subsection{LHCF Training}
\label{sec:LHCF_Training}

The \textbf{LHCF training has two steps} (Fig.~\ref{fig:method_overview}): \textbf{(i)} unsupervised discovery of latent cohorts from image embeddings, and \textbf{(ii)} fairness‑aware training, where the discovered cohort identities are treated as sensitive attributes. 

\vspace{1ex}

\noindent
\textbf{Step (i)} assumes the availability of a labeled dataset  $\mathcal{D}=\{(x_i,y_i)\}_{i=1}^{N}$, containing medical images $x \in \mathcal{X}$ and diagnostic labels $y \in \mathcal{Y} \subseteq \Delta^{|\mathcal{Y}|-1}$, but without demographic annotations.
Our goal is to learn the model 
$f_{\theta} = d_{\theta_d} \circ e_{\theta_e}$, where \(\theta = (\theta_e, \theta_d)\), 
$e_{\theta_e} : \mathcal{X} \to \mathcal{Z}$ represents the encoder, and
$d_{\theta_d} : \mathcal{Z} \to \mathcal{Y}$ denotes the decoder. We assume that the model $f_{\theta}(.)$ has been pre-trained either with self-supervised learning~\cite{jing2020self} or with supervised learning using $\mathcal{D}$.
In this Step (i), the method first finds the hidden cohorts via Gaussian Mixture Model (GMM)~\cite{dempster1977maximum}, where the number of clusters is estimated with the Bayesian Information Criterion (BIC)~\cite{schwarz1978estimating}.
This clustering happens with the image embeddings in $\mathcal{Z} = \{ z_i | z_i=e_{\theta_e}(x_i), (x_i,y_i) \in \mathcal{D}\}$.
For $K \in \mathbb{N}$, a Gaussian mixture is defined by 
$p(z\mid\Phi_K)=\sum_{k=1}^K \pi_k\,\mathcal{N}(z\mid \mu_k,\Sigma_k), 
\quad \Phi_K=\{\pi_k,\mu_k,\Sigma_k\}_{k=1}^K$,
with log-likelihood $\ell(\Phi_K;\mathcal{Z})=\sum_{i=1}^N \log \sum_{k=1}^K \pi_k \mathcal{N}(z_i\mid\mu_k,\Sigma_k)$.
We estimate $\Phi^*_K$ by Expectation Maximization (EM)~\cite{dempster1977maximum}, yielding the responsibility $\gamma_{ik}=p(k\mid z_i,\Phi^*_K)=
\frac{\pi_k \mathcal{N}(z_i\mid \mu_k,\Sigma_k)}{\sum_{j=1}^K \pi_j \mathcal{N}(z_i\mid \mu_j,\Sigma_j)}$. The number of cohorts can be estimated with
\begin{equation}
    \textstyle K^*\in \operatorname*{argmin}_{K\in\{1,\dots,N\}} \mathrm{BIC}(K) = \operatorname*{argmin}_{K\in\{1,\dots,N\}} p_K \log N - 2\,\ell(\Phi^*_K;Z),
\end{equation}
%\textcolor{red}{could we replace \(K_{min}\) by 1 and \(K_{max}\) by N because there are maximum \(N\) samples?}
where $p_K=(K-1)+Kd+K\,\frac{d(d+1)}{2}$ for full covariances. The hidden cohort dataset to be used in subsequent fairness optimization is then defined by $\widehat{\mathcal{D}} = \{ (x_i,y_i,c_i) | (x_i,y_i) \in \mathcal{D}, c_i=\arg\max_{k\in\{1,\dots,K^*\}} \gamma_{ik} \}$, where $c_i$ represents the hidden cohort label of the sample $(x_i,y_i)$ computed with the responsibility $\gamma_{ik}$.

\vspace{1ex}

\noindent
\textbf{Step (ii)} uses the hidden cohort dataset $\widehat{\mathcal{D}}$ to optimize
\begin{equation}
\label{eq:objective_loss}
    \textstyle \theta^{*} = \operatorname*{argmin}_{\theta}\;\frac{1}{N}\sum_{i=1}^{N}\ell_{\mathrm{clss}}\!\big(y_i,f_{\theta}(x_i)\big)
+\lambda\,\ell_{\mathrm{fair}}\!\big(\widehat{\mathcal{D}},\theta\big),
\end{equation}
where $\lambda\!\ge\!0$ is a hyper-parameter, $\ell_{\mathrm{clss}}(.)$ is a classification loss and $\ell_{\mathrm{fair}}(.)$ denotes a fairness loss that can be defined in different ways, such as:
\begin{equation}
    \textstyle \ell_{\mathrm{fair}}^{\mathrm{wst}}\big(\widehat{\mathcal{D}},\theta\big) = \max_{k} R_k(\theta)
    \quad\text{or}\quad
    \ell_{\mathrm{fair}}^{\mathrm{gap}}\big(\widehat{\mathcal{D}},\theta\big) = \max_{k} R_k(\theta) -\min_{k} R_k(\theta),
\end{equation}
with $R_k(\theta)=\frac{1}{|C_k|} \sum_{(x_{i}, y_{i}) \in C_{k}} \ell_{\mathrm{clss}}\!\big(y_i,f_{\theta}(x_i)\big)$
% ($S_k=\{i:\,c_i=k\}$)
denoting classification loss of samples in the cluster \(C_{k}\), $\ell_{\mathrm{fair}}^{\mathrm{wst}}(\cdot)$ minimizing the worst‑cohort classification loss, and $\ell_{\mathrm{fair}}^{\mathrm{gap}}(\cdot)$ minimizing the gap between the best‑ and worst‑cohort losses. 
LHCF is agnostic to the model and fairness loss, which together define a FairAI method (e.g., SWAD~\cite{cha2021swad}, FIS\cite{luo2024fairvisionequitabledeeplearning}, FEBS~\cite{tian2024fairseg}, FairCLIP~\cite{luo2024fairclip}, FaMI~\cite{tian2024fairdomain}, FairDi~\cite{masroor2024fair}).

\vspace{1ex}

\noindent\textbf{Theoretical analysis}
We show that under mild clustering assumptions, the fairness loss 
$\ell_{\mathrm{fair}}^{\mathrm{wst}}(\widehat{\mathcal{D}},\theta)$ 
upper-bounds the loss of any visible cohort; thus, minimizing it yields fairness across all visible cohorts~\citep{sagawa2019distributionally}. 
Specifically, we assume that the GMM with hard assignments partitions the data into 
$K$ disjoint clusters $\{C_1,\dots,C_K\}$ with $C_k \cap C_{k'} = \varnothing$ for $k \neq k'$, 
and that each visible cohort $G_t$ is a union of these clusters, i.e., $G_t = \bigcup_{k \in \mathcal{K}} C_k$, where 
$\mathcal{K} \subseteq \{1,\dots,K\}$.
%We prove that under certain clustering conditions, the proposed approach, and in particular the fairness loss \(\ell_{\mathrm{fair}}^{\mathrm{wst}}\big(\widehat{\mathcal{D}},\theta\big)\) is an upper-bound of the loss of any visible cohorts. Thus, minimizing this fairness loss will result in a model that is fair across all visible cohorts.

% \noindent
% We assume that the GMM with hard assignment partitions the whole data space into \(K\) separated clusters: \(\{C_{1}, \dots, C_{K}\}\)
% % with these clusters being either non-overlapping or overlapping \(C_{k} \cap C_{k'} = \varnothing \vee C_{k} \cap C_{k'} \neq \varnothing, \forall k \neq k'\).
% with \(C_{k} \cap C_{k'} = \varnothing, \forall k \neq k'\)
% For simplicity, we also assume that each visible cohort \(G_{i}\) is a union of hidden clusters. This means that there exists a set of indices \(\mathcal{K} \subseteq \{1, \dots, K\}\), such that: \(G_{i} = \cup_{k \in \mathcal{K}} C_{k}\).
\vspace{-7pt}
\begin{lemma}
    The empirical risk of any visible cohort is upper-bounded by 
    %the worst fairness loss 
    \(\ell_{\mathrm{fair}}^{\mathrm{wst}} ( \widehat{\mathcal{D}}, \theta )\), meaning:
    %\begin{equation*}
        \(\textstyle \mathsf{L}(G_{t}) = \frac{1}{| G_{t} |} \sum_{(x_{i}, y_{i}) \in G_{t}} \ell_{\text{clss}} (y_{i}, f_{\theta}(x_{i})) \le \max_{k} R_{k}(\theta).\)
    %\end{equation*}
\end{lemma}
\vspace{-7pt}
\begin{proof}
    %From the assumption 
    Assuming \(G_{t} = \cup_{k \in \mathcal{K}} C_{k}\), \(\mathbb{I}((x_{i}, y_{i}) \in G_{t}) = \sum_{k \in \mathcal{K}} \mathbb{I}((x_{i}, y_{i}) \in C_{k})\), where \(\mathbb{I}(.)\) is the indicator function, 
    %The risk on a visible cohort \(G_{i}\) can be re-written as follows:
    we re-write the risk on \(G_{t}\) as:
    \begin{equation}
        \begin{aligned}[b]
            \mathsf{L}(G_{t}) & = \textstyle \frac{1}{| G_{t} |} \sum_{(x_{i}, y_{i})} \ell_{\text{clss}} (y_{i}, f_{\theta}(x_{i})) \mathbb{I}((x_{i}, y_{i}) \in G_{t}) \\
            & \textstyle = \frac{1}{| G_{t} |} \sum_{(x_{i}, y_{i})} \sum_{k \in \mathcal{K}} \ell_{\text{clss}} (y_{i}, f_{\theta}(x_{i})) \mathbb{I}((x_{i}, y_{i}) \in C_{k})\\
            & \textstyle = \frac{1}{| G_{t} |} \sum_{k \in \mathcal{K}} \sum_{(x_{i}, y_{i}) \in C_{k}} \ell_{\text{clss}} (y_{i}, f_{\theta}(x_{i})) \\
            & \textstyle = \frac{1}{| G_{t} |} \sum_{k \in \mathcal{K}} | C_{k} | R_{k}(\theta) 
            \le \sum_{j \in \mathcal{K}} \nicefrac{|C_{j} |}{|G_{t}|} \max_{k} R_{k}(\theta) \le \max_{k} R_{k}(\theta).
        \end{aligned}
    \end{equation}
\end{proof}\qed
\vspace{-15pt}

\subsection{HIDFairBench Benchmark}
\label{sec:benchmark}

% We benchmark our approach using a comprehensive experimental protocol designed to evaluate both predictive performance and fairness across hidden and visible cohorts. This subsection details the datasets used in our study, the evaluation metrics for accuracy and fairness, the range of FairAI strategies and learning paradigms considered, and the statistical testing procedures employed to assess the significance and robustness of the results.

%This section summarizes the datasets, metrics, FairAI methods, and statistical tests used to evaluate accuracy and fairness across hidden and visible cohorts.

%This section presents the benchmark’s datasets, metrics, and FairAI methods.

\subsubsection{Datasets.} 
% We evaluate our approach on three public medical imaging datasets: Fitzpatrick17K~\cite{groh2021evaluating}, HAM10000~\cite{tschandl2018ham10000}, and CMMD~\cite{cai2023online}.
% HAM10000 is a dermoscopic skin lesion dataset containing 9,948 images across seven diagnostic categories, which we follow prior work~\cite{maron2019systematic} in grouping into \emph{benign} and \emph{malignant} classes. We consider multiple visible cohort partitions, including \emph{Gender} (Female, Male), \emph{Age\_binary} ($\leq$60, $>$60), \emph{Age\_multi4} ($\leq$39, 40--59, 60--79, $\geq$80), as well as intersectional cohorts formed by \emph{Gender $\times$ Age\_binary} and \emph{Gender $\times$ Age\_multi4}.
% For Fitzpatrick17K, we binarize lesion labels by grouping \emph{non-neoplastic} and \emph{benign} cases as benign, and \emph{malignant} cases as malignant. Fitzpatrick skin type labels (types 0--5) are used as multi-group visible demographic attributes.
% CMMD consists of 5,152 mammography images annotated with diagnostic labels and patient age. We frame the task as \emph{non-cancer} versus \emph{cancer} classification and use a binarized age attribute ($\leq$55, $>$55) as the sensitive attribute.

The benchmark relies on three public medical imaging datasets: \textit{Fitzpatrick17K}~\cite{groh2021evaluating}, \textit{HAM10000}~\cite{tschandl2018ham10000}, and \textit{CMMD}~\cite{cai2023online}.
For HAM10000 (9,948 dermoscopic images), we follow prior work~\cite{maron2019systematic} and binarize diagnoses into benign vs. malignant. Visible cohort partitions include \emph{Gender}, \emph{Age\_binary} ($\leq$60, $>$60), \emph{Age\_multi4} ($\leq$39, 40--59, 60--79, $\geq$80), and intersectional \emph{Gender $\times$ Age\_binary} and \emph{Gender $\times$ Age\_multi4} groups.
For Fitzpatrick17K, we group lesions into benign (non-neoplastic + benign) vs. malignant, and use Fitzpatrick \emph{skin types} (0–5) as visible demographic attributes. For both HAM10000 and Fitzpatrick17K, we use an 80\%/10\%/10\% train/validation/test split.
For CMMD (5,152 mammography images), we frame classification as non‑cancer vs. cancer, use binarized \emph{patient age} ($\leq$55, $>$55)  as the sensitive attribute, and rely on a 70\%/10\%/20\% train/validation/test split. 

% \textcolor{red}{Do we follow the train/val/test partition from somewhere? We should mention this here.} 
% \textcolor{blue}{Yes. For HAM10000 and Fitzpatrick17K, we follow the official MedFair benchmark train/validation/test partition (80\%/10\%/10\%). For CMMD, we use a 70\%/10\%/20\% train/validation/test split.}
\vspace{1ex}
\noindent
\textbf{Evaluation \& Statistical Tests.}
% We employ distinct evaluation measures for hidden cohort discovery and fairness assessment. To evaluate the quality of discovered hidden cohorts, we adopt two metrics: the Brier Score (BS)~\cite{olesen2024slicing}, a proper scoring rule that quantifies calibration error between predicted confidence and binary outcomes and highlights under- and overperforming cohorts; and Average Purity (AP)~\cite{bissoto2025subgroup}, which measures cohort cohesion by assessing alignment between discovered cohorts and available demographic attributes such as gender, age, or skin type. For fairness evaluation, classification performance is measured using AUC, while disparity is quantified using four complementary metrics: AUC Gap and Worst-Case AUC~\cite{zong2022medfair}, Equity-Scaled AUC (ES-AUC)~\cite{luo2024fairvisionequitabledeeplearning}, and Performance-Scaled Disparity (PSD), which relates cohort-level performance to overall accuracy. Statistical significance is assessed following~\cite{zong2022medfair} using the Friedman test~\cite{friedman1937use}, with the Nemenyi post-hoc test~\cite{nemenyi1963distribution} applied when $p<0.05$; results are visualized using Critical Difference diagrams, where methods connected by a line are not significantly different. In addition, we perform paired one-sided $t$-tests across multiple runs to directly assess whether observed improvements between competing FairAI methods are statistically significant.
The hidden-cohort quality is measured using the \emph{Brier Score (BS)}~\cite{olesen2024slicing}, which captures calibration and identifies under- and overperforming cohorts, and \emph{Average Purity (AP)}~\cite{bissoto2025subgroup}, which quantifies alignment between discovered cohorts and visible demographic attributes. 
For visible cohort evaluation, we report \emph{AUC} and four fairness metrics: \emph{AUC Gap}, \emph{Worst-Case AUC}~\cite{zong2022medfair}, \emph{ES-AUC}~\cite{luo2024fairvisionequitabledeeplearning}, and \emph{Performance-Scaled Disparity (PSD)}.
Statistical significance uses the  \emph{Friedman test}~\cite{friedman1937use,zong2022medfair} with \emph{Nemenyi} post-hoc comparisons~\cite{nemenyi1963distribution} ($p<0.05$), visualized via Critical Difference (CD) diagrams. 
%Additionally, paired one-sided $t$-tests across runs assess significance between competing methods.

\vspace{1ex}
\noindent
\textbf{FairAI Methods \& Backbones.}
The benchmark tests many \textit{FairAI methods}, starting from the
\emph{ERM}~\cite{vapnik1999overview} baseline, which optimizes average classification error without addressing fairness.
\emph{SWAD}~\cite{cha2021swad} improves generalization by promoting convergence to flat minima.
Among fairness approaches, \emph{FIS}~\cite{luo2024fairvisionequitabledeeplearning} reweights group losses to equalize performance, while \emph{FEBS}~\cite{tian2024fairseg} emphasizes hard examples via error-bound scaling.
We also include \emph{FairCLIP-OT}, an optimal-transport fairness regularizer adapted from FairCLIP~\cite{luo2024fairclip}, and \emph{FaMI}~\cite{tian2024fairdomain}, which reduces dependence between representations and sensitive attributes via mutual information.
Furthermore, \emph{FairDi}~\cite{masroor2024fair} learns a shared fair backbone and distills knowledge from cohort‑specific experts into a unified student model.
For the \textit{backbone models}, we consider \textit{ResNet18}~\cite{he2016deep}, \textit{CLIP}~\cite{fang2023eva},  \textit{MedCLIP}~\cite{wang2022medclip}, and \textit{DINOv2}~\cite{oquab2023dinov2}.

\section{Experiments}
\label{sec:experiments}

%This section presents the HIDFairBench results and ablation studies.

\subsection{HIDFairBench Results}
\label{sec:HIDFairBench_results}

\begin{table}[t]
    \centering
\caption{AUC and BS across hidden clusters (Clt. ID) on HAM10000 (ERM, ResNet18), with cohort proportions (\%) of malignant/benign, male/female, and age $\le60$ or $>60$ years old.}
\label{tab:hidden_cohort_risk}

\scalebox{0.75}{
\begin{tabular}{c c @{\hspace{3ex}} c @{\hspace{1em}} c @{\hspace{1em}} c @{\hspace{1em}} c @{\hspace{1em}} c}
\toprule
\textbf{Clt.ID} & \textbf{AUC} & \textbf{BS} &  \textbf{(Mal)(Ben) } & \textbf{(Male)(Female) } & \textbf{($\le$60)($>$60) } \\
\midrule
0 & -- & 5.3E-5 & (0.00)(100.00) & (48.66)(51.34) & (78.61)(21.39) \\
1 & 0.6179 & 0.1506 & (18.55)(81.45) & (39.18)(60.82) & (70.10)(29.90) \\
2 & 0.8247 & 0.0241 & (2.63)(97.37) & (46.24)(53.76) & (74.44)(25.56) \\
3 & 0.7507 & 0.1759 & (27.27)(72.73) & (45.45)(54.55) & (54.54)(45.46) \\
4 & 0.6113 & 0.2255 & (36.36)(63.64) & (63.64)(36.36) & (35.06)(64.94) \\
5 & 0.7557 & 0.1090 & (13.65)(86.35) & (59.03)(40.97) & (56.38)(43.62) \\
6 & 0.7299 & 0.1956 & (66.39)(33.61) & (69.75)(30.25) & (23.53)(76.47) \\
\bottomrule
\end{tabular}
}
\end{table}

\subsubsection{Hidden-Cohort Quality.}
Using an ERM-trained ResNet18, BS on HAM10000 reveals a clear risk stratification across hidden cohorts (Table~\ref{tab:hidden_cohort_risk}): low malignancy-rate cohorts (Clusters~0,2) show low BS, intermediate  malignancy-rate cohorts (1,3,5) exhibit moderate BS, and high  malignancy-rate cohorts (4,6) attain the highest BS, indicating reduced reliability on clinically severe cases. These trends mirror underlying population structure, with cohorts dominated by older patients showing higher malignancy prevalence, consistent with known age-related risk patterns in dermatology.
Across CMMD and Fitzpatrick17K, BS similarly tracks malignancy severity, with low malignancy-rate cohorts obtaining low calibration error and high malignancy-rate  clusters showing higher BS. Despite this clear risk stratification, hidden cohorts show weak alignment with available demographic metadata: on HAM10000, AP is $0.4919 {\pm} 0.0183$ (gender) and $0.2234 {\pm} 0.0065$ (age); on Fitzpatrick17K, AP for skin type is $0.0998 {\pm} 0.0161$; and on CMMD, AP for age is $0.3517{\pm}0.0067$.  Such APs $<0.5$ confirm that hidden cohorts primarily reflect visual structure rather than demographic attributes. 

\noindent
\textbf{Visible Cohort Evaluation.}
% Table~\ref{tab:visible_hidden_cohort_fairness_results} compares visible-demographics supervised (\emph{Classic}) and hidden-cohort-based (\emph{LHCF}) training in terms of accuracy and fairness measures across all datasets and demographic attributes. LHCF consistently improves both accuracy and fairness results relative to the visible-demographics supervised learning strategy, with gains becoming increasingly pronounced for complex and intersectional cohorts (e.g., \emph{Gender $\times$ Age\_binary}, \emph{Gender $\times$ Age\_multi4}).
% Improvements are smaller for simple binary attributes but remain consistent across methods. Notably, these gains are achieved \emph{without} using demographic metadata during training, highlighting the relevance of hidden-cohort learning in addressing AI fairness. 
Table~\ref{tab:visible_hidden_cohort_fairness_results} compares visible‑demographic supervised (\emph{Classic}) and \emph{LHCF} training across all datasets and demographic attributes. 
%LHCF consistently improves AUC and fairness, with strong gains for complex and intersectional cohorts (see Fig.~\ref{fig:motivation_visible_complexity}).  These improvements are achieved \emph{without} demographic metadata, and despite weak alignment between hidden cohorts and demographic groups (AP$<0.5$), fairness learned from appearance‑based latent cohorts generalizes reliably to visible attributes at test time. Across FairAI methods, LHCF generally outperforms classic demographic‑supervised training on visible cohorts, with FairDi+LHCF delivering the strongest overall results.
LHCF improves AUC and fairness, especially for intersectional cohorts, without using any demographic metadata (see Fig.~\ref{fig:motivation_visible_complexity}). Despite weak alignment between hidden and demographic groups (AP$<0.5$), the fairness learned from appearance‑based cohorts generalizes well to visible attributes, with FairDi+LHCF delivering the strongest overall results.
The CD diagrams in Fig.~\ref{fig:cd_diagrams_visible_cohorts}  confirm the superiority of LHCF, and in particular the FairDi+LHCF as the top‑ranked approach across nearly all evaluation metrics.
%\textcolor{red}{An important question is why LHCF, trained solely on hidden cohorts, can outperform classic methods trained and tested directly on visible cohorts. We argue that hidden cohorts provide a stronger and more informative training signal, as  illustrated in Fig.~\ref{fig:improvement_wrt_erm}, where FairDi achieves substantially larger improvements (for Overall AUC, Min. AUC, ES-AUC, AUC Gap) over ERM when trained and tested on the hidden cohorts than on the visible cohorts (Gender) of HAM10000.}
An important question is why LHCF, trained on hidden cohorts, outperforms classical methods trained on visible cohorts. The intuition, based on Lemma 1, is that instead of minimizing the average loss over overlapping visible cohorts, which can overemphasize intersectional groups, LHCF minimizes the loss over hidden cohorts, thereby reducing extreme disparities and producing more balanced performance.

%\cellcolor{myred}{Update the table results with methods trained with classic and hidden cohort approaches.} \textcolor{blue}{I updated Table~2 accordingly. Note that the results for ERM remain identical in both the Classic and LHCF settings. This is expected because ERM does not incorporate any fairness regularization; it is trained solely using the standard BCE loss. Therefore, changing the training paradigm (visible-supervised vs. hidden-cohort learning) does not affect its objective, leading to identical performance. Additionally, we observe that FaMI performs better in the Classic setting compared to LHCF. This suggests that FaMI is more effective when fairness regularization is directly aligned with the visible demographic attributes. If you think it is better to remove it, we can do that. However, I believe we should keep it. Because we are introducing a new benchmark.}

\begin{table*}[t]
\centering
\caption{Classic vs. LHCF training results averaged over HAM10000 (\emph{Gender}, \emph{Age\_binary}, \emph{Age\_multi4}, \emph{Gender $\times$ Age\_binary}, \emph{Gender $\times$ Age\_multi4}), Fitzpatrick17K (\emph{Skin Type}), and CMMD (\emph{Age\_binary}). Cell colors denote \colorbox{mygreen}{Better}/\colorbox{myred}{Worse}/\colorbox{myyellow}{Same} between LHCF and Classic, and \textbf{Bold}=best result.}
\label{tab:visible_hidden_cohort_fairness_results}

\scalebox{0.85}{
\fontsize{8}{9}\selectfont
\begin{tabular}{c l c c c c c c c}
\toprule
& \textbf{Measures and Ranks} 
& \textbf{ERM} 
& \textbf{SWAD} 
& \textbf{FIS} 
& \textbf{FaMI} 
& \textbf{FEBS} 
& \textbf{FairCLIP-OT} 
& \textbf{FairDi} \\
\midrule

% ---------- VISIBLE-SUPERVISED ----------
\multirow{6}{*}{\rotatebox{90}{\textbf{Classic}}}
& Avg. Overall AUC Score$\uparrow$ 
& \cellcolor{myyellow}{0.8775} 
& \cellcolor{myred}{0.8867} 
& \cellcolor{mygreen}{0.8806} 
& \cellcolor{myred}{0.8582} 
& \cellcolor{myred}{0.8731} 
& \cellcolor{myred}{0.8692} 
& \cellcolor{myred}{0.9014} \\

& Avg. Min. AUC Score$\uparrow$ 
& \cellcolor{myyellow}{0.8213} 
& \cellcolor{myred}{0.8340} 
& \cellcolor{myred}{0.8212} 
& \cellcolor{mygreen}{0.8113} 
& \cellcolor{myred}{0.8246} 
& \cellcolor{myred}{0.8083} 
& \cellcolor{myred}{0.8511} \\

& Avg. ES-AUC Score$\uparrow$ 
& \cellcolor{myyellow}{0.8447} 
& \cellcolor{myred}{0.8578} 
& \cellcolor{myred}{0.8477} 
& \cellcolor{myred}{0.8304} 
& \cellcolor{myred}{0.8444} 
& \cellcolor{myred}{0.8374} 
& \cellcolor{myred}{0.8744} \\

& Avg. Gap AUC Score$\downarrow$ 
& \cellcolor{myyellow}{0.1053} 
& \cellcolor{myred}{0.0905} 
& \cellcolor{myred}{0.1019} 
& \cellcolor{mygreen}{0.0953} 
& \cellcolor{myred}{0.0940} 
& \cellcolor{myred}{0.1012} 
& \cellcolor{myred}{0.0824} \\

& Avg. Mean PSD Score$\downarrow$ 
& \cellcolor{myyellow}{0.0478} 
& \cellcolor{myred}{0.0406} 
& \cellcolor{myred}{0.0487} 
& \cellcolor{mygreen}{0.0448} 
& \cellcolor{mygreen}{0.0444} 
& \cellcolor{myred}{0.0473} 
& \cellcolor{myred}{0.0354} \\

& Avg. Max PSD Score$\downarrow$ 
& \cellcolor{myyellow}{0.1220} 
& \cellcolor{myred}{0.1042} 
& \cellcolor{myred}{0.1186} 
& \cellcolor{mygreen}{0.1134} 
& \cellcolor{myred}{0.1104} 
& \cellcolor{myred}{0.1192} 
& \cellcolor{myred}{0.0915} \\

\midrule

% ---------- HIDDEN-COHORT ----------
\multirow{6}{*}{\rotatebox{90}{\textbf{LHCF}}}
& Avg. Overall AUC Score$\uparrow$ & \cellcolor{myyellow}{0.8775} & \cellcolor{mygreen}{0.8889} & \cellcolor{myred}{0.8804} & \cellcolor{mygreen}{0.8668} & \cellcolor{mygreen}{0.8827} & \cellcolor{mygreen}{0.8703} & \cellcolor{mygreen}{\textbf{0.9050}} \\
& Avg. Min. AUC Score$\uparrow$ & \cellcolor{myyellow}{0.8213} & \cellcolor{mygreen}{0.8418} & \cellcolor{mygreen}{0.8266} & \cellcolor{myred}{0.8106} & \cellcolor{mygreen}{0.8369} & \cellcolor{mygreen}{0.8127} & \cellcolor{mygreen}{\textbf{0.8686}} \\
& Avg. ES-AUC Score$\uparrow$ & \cellcolor{myyellow}{0.8447} & \cellcolor{mygreen}{0.8617} & \cellcolor{mygreen}{0.8486} & \cellcolor{mygreen}{0.8339} & \cellcolor{mygreen}{0.8519} & \cellcolor{mygreen}{0.8399} & \cellcolor{mygreen}{\textbf{0.8817}} \\
& Avg. Gap AUC Score$\downarrow$ & \cellcolor{myyellow}{0.1053} & \cellcolor{mygreen}{0.0843} & \cellcolor{mygreen}{0.0967} & \cellcolor{myred}{0.1061} & \cellcolor{mygreen}{0.0929} & \cellcolor{mygreen}{0.0963} & \cellcolor{mygreen}{\textbf{0.0666}} \\
& Avg. Mean PSD Score$\downarrow$ & \cellcolor{myyellow}{0.0478} & \cellcolor{mygreen}{0.0386} & \cellcolor{mygreen}{0.0461} & \cellcolor{myred}{0.0498} & \cellcolor{myred}{0.0450} & \cellcolor{mygreen}{0.0451} & \cellcolor{mygreen}{\textbf{0.0312}} \\
& Avg. Max PSD Score$\downarrow$ & \cellcolor{myyellow}{0.1220} & \cellcolor{mygreen}{0.0976} & \cellcolor{mygreen}{0.1124} & \cellcolor{myred}{0.1245} & \cellcolor{mygreen}{0.1080} & \cellcolor{mygreen}{0.1132} & \cellcolor{mygreen}{\textbf{0.0762}} \\
\bottomrule

\end{tabular}
}
\end{table*}

\begin{figure*}[t]
    \centering

    % ---------------- Row 1 ----------------
    \begin{subfigure}[t]{0.32\textwidth}
        \centering
        \includegraphics[width=\linewidth]{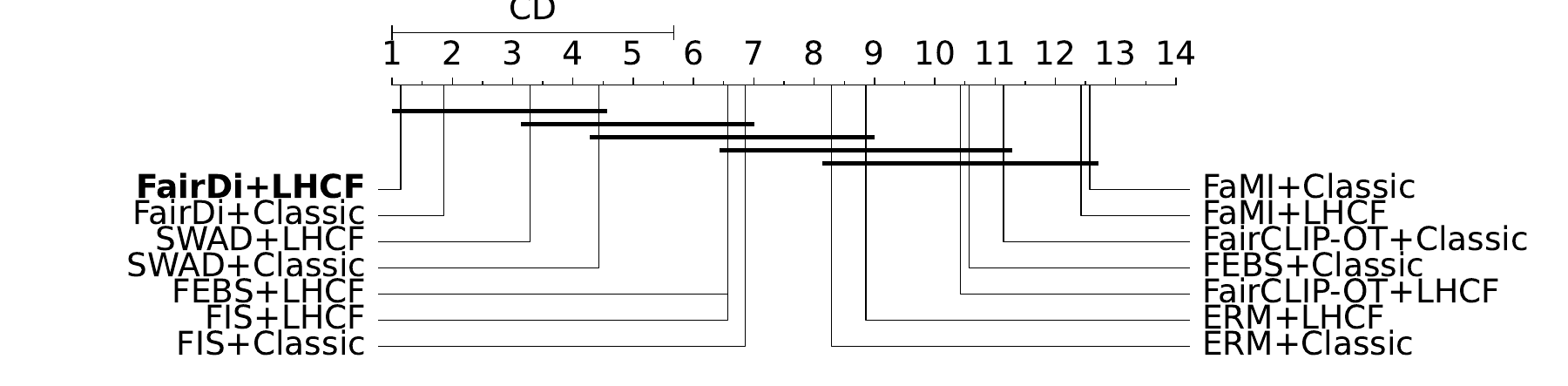}
        \caption{Overall AUC}
        \label{fig:cd_overall_auc_visible}
    \end{subfigure}
    \hfill
    \begin{subfigure}[t]{0.32\textwidth}
        \centering
        \includegraphics[width=\linewidth]{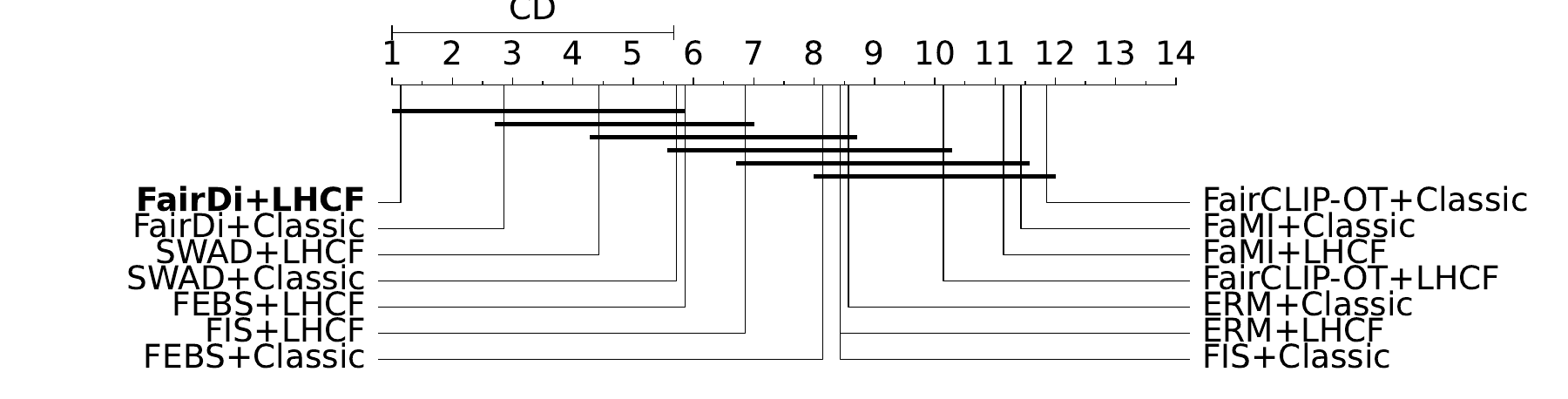}
        \caption{Min. AUC}
        \label{fig:cd_worst_auc_visible}
    \end{subfigure}
    \hfill
    \begin{subfigure}[t]{0.32\textwidth}
        \centering
        \includegraphics[width=\linewidth]{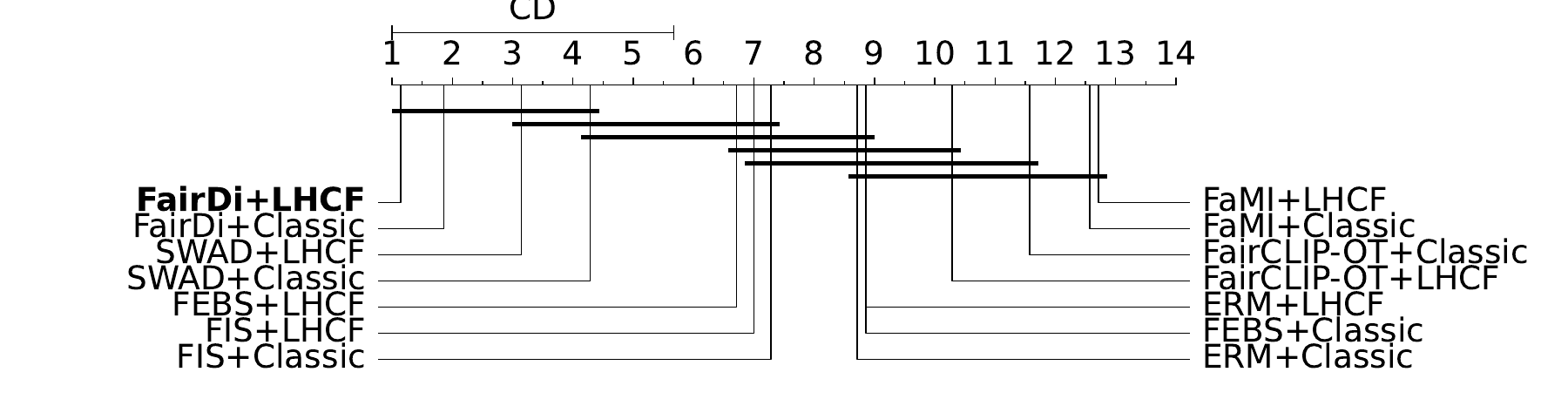}
        \caption{ES-AUC}
        \label{fig:cd_es_auc_visible}
    \end{subfigure}

    \vspace{2mm}

    % ---------------- Row 2 ----------------
    \begin{subfigure}[t]{0.32\textwidth}
        \centering
        \includegraphics[width=\linewidth]{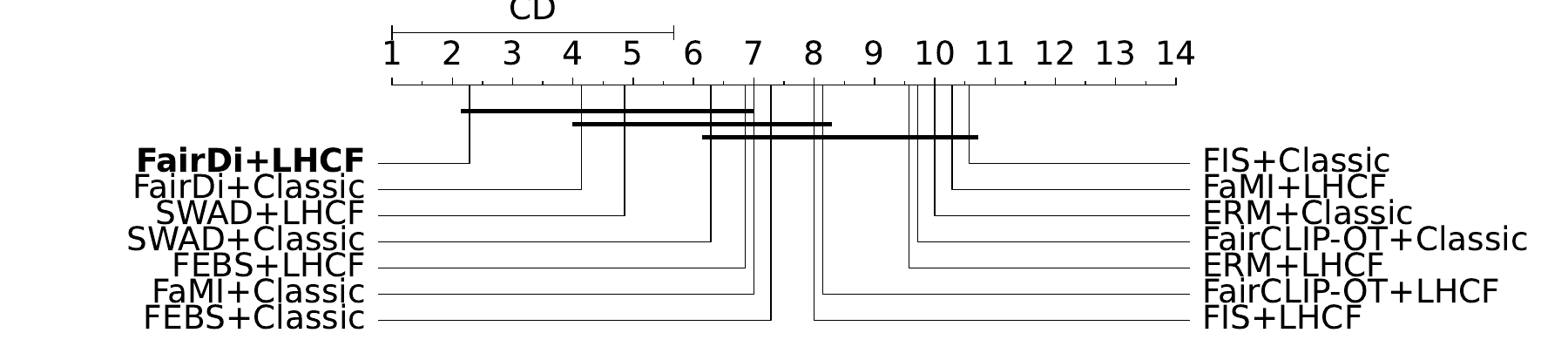}
        \caption{Gap AUC}
        \label{fig:cd_gap_auc_visible}
    \end{subfigure}
    \hfill
    \begin{subfigure}[t]{0.32\textwidth}
        \centering
        \includegraphics[width=\linewidth]{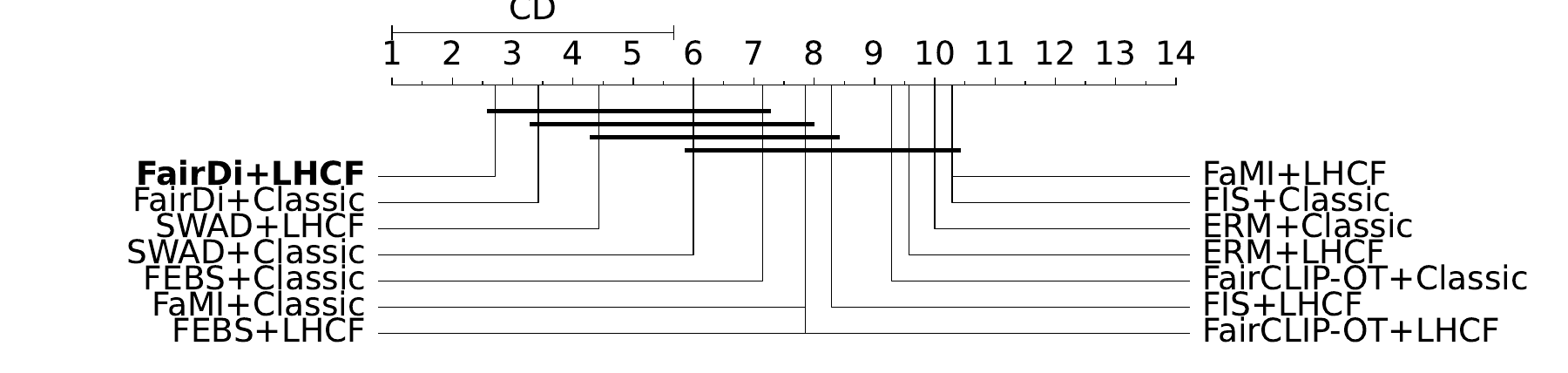}
        \caption{Mean PSD}
        \label{fig:cd_mean_psd_visible}
    \end{subfigure}
    \hfill
    \begin{subfigure}[t]{0.32\textwidth}
        \centering
        \includegraphics[width=\linewidth]{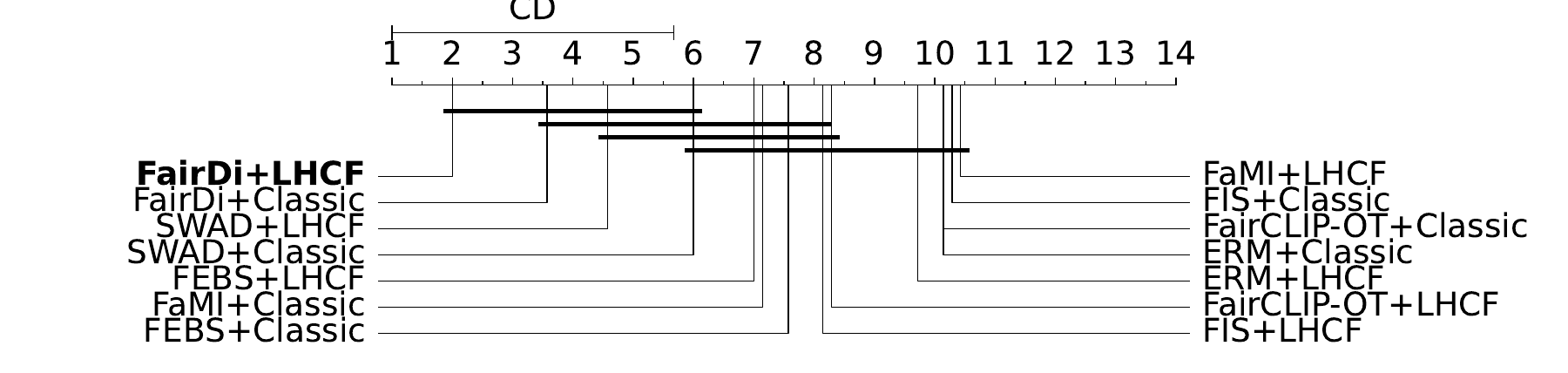}
        \caption{Max PSD}
        \label{fig:cd_max_psd_visible}
    \end{subfigure}

    \caption{
    CD diagrams comparing Classic and LHCF learning across all datasets and demographic cohorts (Please zoom in to view the figures clearly).
    }
    \label{fig:cd_diagrams_visible_cohorts}
\end{figure*}

\subsection{Ablation Studies}

\noindent
\textbf{Number of Hidden Cohorts ($K$).}
Fig.~\ref{fig:k_ablation_ham10000} evaluates the performance of LHCF as a function of the number of cohorts.
The BIC choice (i.e., $K{=}7$) produces the best results for AUC and fairness measures for the majority of visible cohort partitions, whereas larger or smaller values of $K\in\{3,5,9\}$ tend to degrade performance. 
This supports BIC as an effective criterion for selecting $K$.

\begin{figure*}[t]
    \centering

    % % -------- Row 1 (a) Hidden Cohorts --------
    % \begin{subfigure}[t]{0.9\textwidth}
    %     \centering
    %     \includegraphics[width=\linewidth]{Figures/Fairness vs. Number of Hidden Clusters/hidden-cohorts/FairDi_FairDi_K_Ablation.pdf}
    %     \caption{}
    %     \label{fig:k_ablation_hidden}
    % \end{subfigure}

    % \vspace{3mm}

    % -------- Row 2 (b) Visible Cohort Complexity --------
    % \begin{subfigure}[t]{0.9\textwidth}
    %     \centering
        \includegraphics[width=0.9\linewidth]{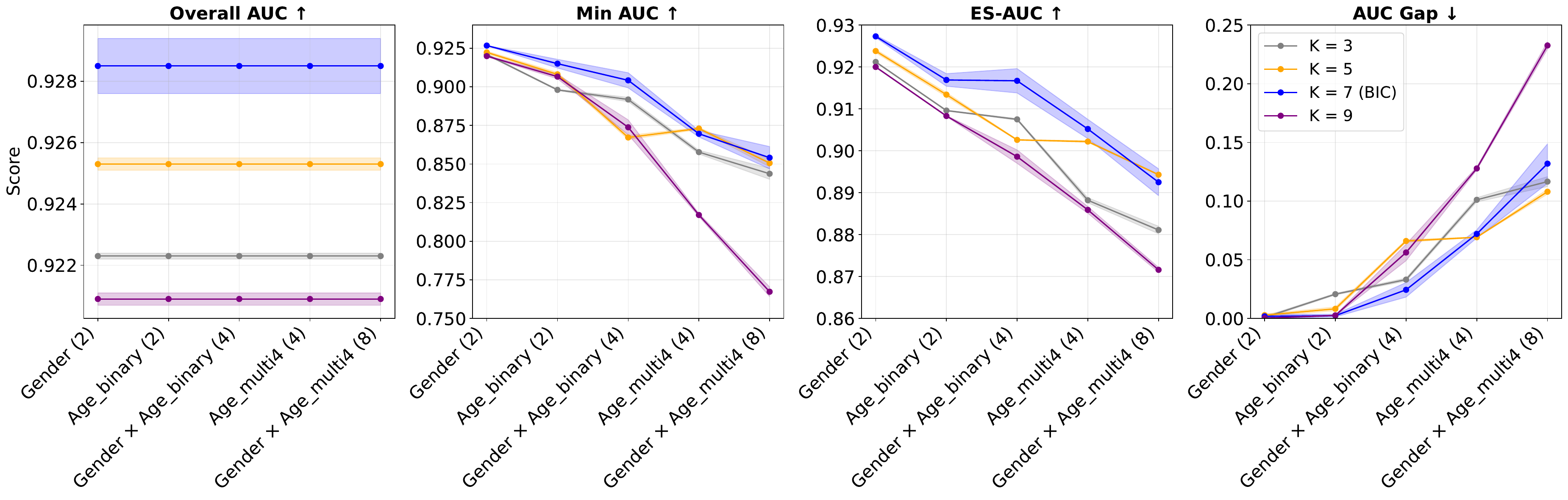}
    %     \caption{}
    %     \label{fig:k_ablation_visible}
    % \end{subfigure}

    \caption{
    Sensitivity analysis of LHCF (FairDi+ResNet18 backbone) with respect to the number of hidden cohorts ($K$) on the HAM10000 dataset. 
    %The analysis is based on the performance measured on visible demographic cohorts of varying complexity (\emph{Gender}, \emph{Age\_binary}, \emph{Age\_multi4}, \emph{Gender $\times$ Age\_binary}, \emph{Gender $\times$ Age\_multi4}).
    }
    \label{fig:k_ablation_ham10000}
\end{figure*}

\vspace{1ex}
\noindent
\textbf{Embedding Backbone.}
%Fig.~\ref{fig:backbone_ablation_ham10000} shows how different image representation backbones (ResNet18, CLIP, MedCLIP, DINOv2) affect the generalization of hidden-cohort fairness to visible demographic groups. ResNet18 and MedCLIP show the highest Overall AUC results. In contrast, fairness metrics show clearer dependence on the induced latent cohort structure: ResNet18 and MedCLIP generally achieve higher Min.~AUC and ES-AUC and maintain lower AUC Gap as cohort complexity increases, whereas CLIP and DINOv2 exhibit larger drops in worst-case performance and higher disparity for intersectional groups. These trends suggest that fairness robustness depends on how the backbone shapes hidden-cohort granularity, with moderate clustering complexity yielding more stable transfer of fairness to visible demographic cohorts.
%Fig.~\ref{fig:backbone_ablation_ham10000} summarizes how representation backbones affect LCHF's performance. ResNet18 and MedCLIP achieve the best Overall AUC, while fairness metrics show clearer separation: for Min.~AUC and AUC Gap, ResNet18 performs best on the simpler intersectional settings; for ES-AUC, ResNet18 is consistently better across all cohort complexities, with MedCLIP a solid second. CLIP and DINOv2 show larger drops in worst-case performance and higher disparity as cohort complexity increases. Overall, backbones inducing moderate hidden-cohort granularity, such as ResNet18 and MedCLIP with $K=7$ and $K=6$ clusters found by BIC, provide better accuracy and fairness results than CLIP and DINOv2 with $K=5$ and $K=10$ clusters found by BIC.
Fig.~\ref{fig:backbone_ablation_ham10000} summarizes how representation backbones affect LHCF. ResNet18 and MedCLIP achieve the best Overall AUC. Fairness metrics separate more clearly: for Min.~AUC and AUC Gap, ResNet18 performs best on simpler intersectional settings; for ES-AUC, ResNet18 is consistently strongest across all cohort complexities, with MedCLIP a solid second. CLIP and DINOv2 show larger drops in worst-case performance and higher disparity as cohort complexity increases. Overall, backbones inducing \emph{moderate} hidden-cohort granularity, e.g., ResNet18 ($K{=}7$) and MedCLIP ($K{=}6$) selected by BIC, yield better accuracy and fairness than CLIP ($K{=}10$) and DINOv2 ($K{=}5$).

\begin{figure*}[t]
    \centering

    % % -------- Row 1 (a) Hidden Cohorts --------
    % \begin{subfigure}[t]{0.9\textwidth}
    %     \centering
         %\includegraphics[width=\linewidth]{Figures/ImageNet Pre-training vs. Medical Foundation Models/hidden-cohorts/FairDi_Backbone_Comparison.pdf}
    %     \caption{}
    %     \label{fig:backbone_ablation_hidden}
    % \end{subfigure}

    % \vspace{3mm}

    % -------- Row 2 (b) Visible Cohort Complexity --------
    % \begin{subfigure}[t]{0.9\textwidth}
    %     \centering
        \includegraphics[width=0.9\linewidth]{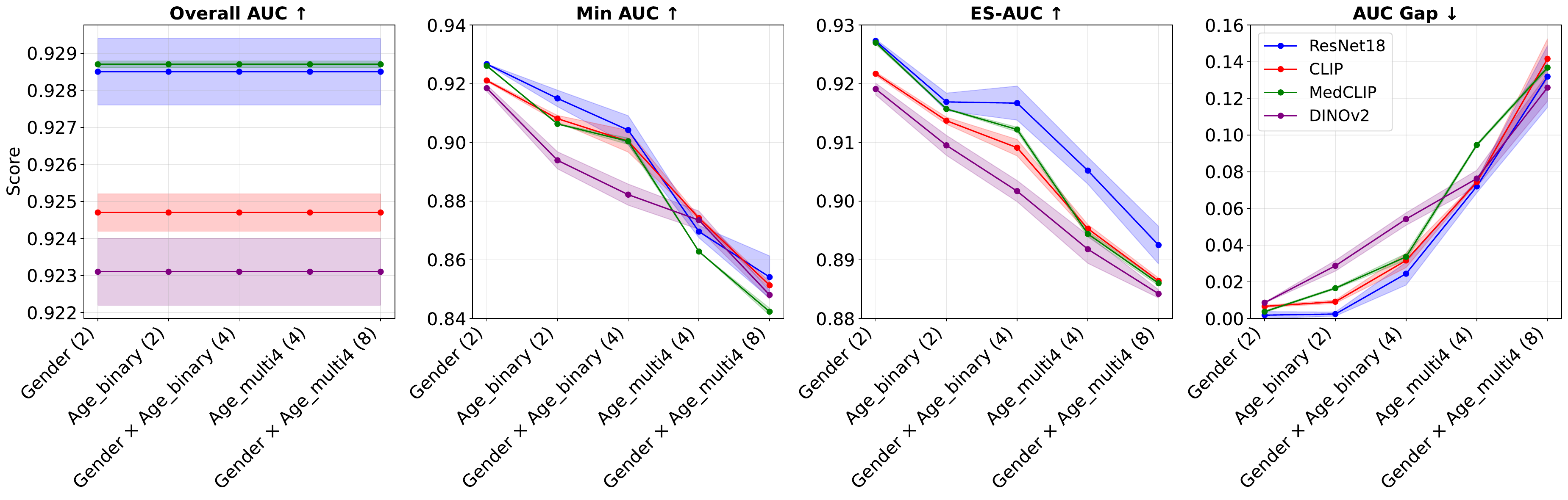}
    %     \caption{}
    %     \label{fig:backbone_ablation_visible}
    % \end{subfigure}

    \caption{
    Sensitivity analysis of LHCF (FairDi) with respect to the  backbone (ResNet18, CLIP~\cite{fang2023eva}, MedCLIP~\cite{wang2022medclip}, and DINOv2~\cite{oquab2023dinov2}) on  HAM10000.
    }
    \label{fig:backbone_ablation_ham10000}
\end{figure*}

\vspace{1ex}
\noindent
\textbf{Demographic-Aware Clustering (DAC) vs. LHCF.}
As an intermediate alternative between \textit{Classic} and \textit{LHCF}, DAC clusters samples using embeddings augmented with demographic attributes and then optimizes fairness over these clusters. Fig.~\ref{fig:dac_vs_lhcf_visible_complexity} shows that LHCF outperforms DAC on nearly all measures and across all cohort complexities. These results indicate that appearance driven, label-free fairness optimization offers greater robustness to visible-cohort complexity than DAC while eliminating any dependence on demographic labels.

\begin{figure*}[t]
    \centering
    \includegraphics[width=0.9\textwidth]{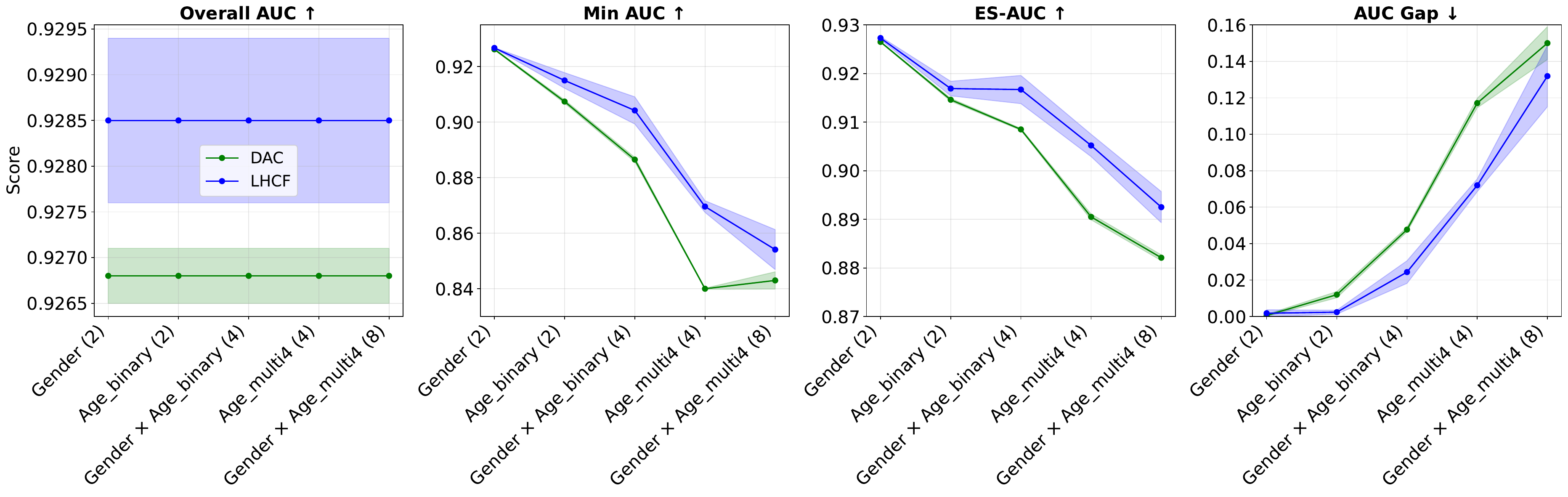}
    \caption{DAC vs. LHCF (FairDi+ResNet18 backbone) on HAM10000 dataset.}
    \label{fig:dac_vs_lhcf_visible_complexity}
\end{figure*}

\section{Conclusion}
\label{sec:con}

\vspace{-5pt}

We introduced LHCF, which optimizes fairness over appearance-based latent cohorts without demographic supervision, and proposed \emph{HIDFairBench} to evaluate fairness under single and multiple demographic attributes. Across three datasets, hidden cohorts exhibit clinical risk stratification (BS rises with malignancy) but weak demographic alignment (AP$<0.5$), indicating they capture meaningful visual structure. Relative to the demographic‑supervised Classic FairAI methods, LHCF improves fairness and accuracy on visible groups without using demographic labels, with solid gains for intersectional cohorts. LHCF is robust to design choices: a BIC-selected moderate number of cohorts ($K{=}7$) yields the best accuracy-fairness tradeoff, and backbones with similar cohort granularity (ResNet18, MedCLIP) provide the best performance. Compared with DAC, LHCF delivers better accuracy and fairness.
In future work, we will improve the reproducibility of LHCF by stabilizing the unsupervised clustering stage and extend Lemma 1 to accommodate more fairness losses; nonetheless, the adoption of data-driven latent stratification offers a more generalizable and deployable pathway toward equitable and trustworthy medical imaging analysis.

%\textcolor{red}{These gains stem from LHCF’s focus on appearance‑based latent cohorts, which expose the visual error modes that drive disparities and provide a stronger training signal than demographic groups (Fig.~\ref{fig:improvement_wrt_erm}). As a result, fairness improvements transfer reliably to visible demographic cohorts even when hidden–visible alignment is limited.} 
%A limitation of LHCF is the reduced reproducibility from the unsupervised clustering stage; nonetheless, the adoption of data-driven latent stratification offers a more generalizable and deployable pathway toward equitable and trustworthy medical imaging analysis.

%A key limitation of our approach lies in the reproducibility of exact results due to the unsupervised clustering stage, as variations in latent cohort discovery can influence downstream fairness outcomes; nevertheless, our findings suggest that embracing data-driven latent stratification offers a more generalizable and practically deployable pathway toward equitable medical imaging systems.

% \bigskip

% \noindent\textbf{Disclosure of Interests.} The authors have no competing interests to declare that are relevant to the content of this article.

%
%
%
%
% ---- Bibliography ----
%
% BibTeX users should specify bibliography style 'splncs04'.
% References will then be sorted and formatted in the correct style.

% \bibliographystyle{splncs04}
% \bibliography{Sections/main}
\printbibliography

\end{document}